\documentclass{article} 
\usepackage{graphicx, booktabs}
\usepackage{multirow}
\usepackage{iclr2026_conference,times}

\usepackage{amsmath,amsfonts,bm}

\def\eqref#1{equation~\ref{#1}}

\def\1{\bm{1}}

\DeclareMathAlphabet{\mathsfit}{\encodingdefault}{\sfdefault}{m}{sl}
\SetMathAlphabet{\mathsfit}{bold}{\encodingdefault}{\sfdefault}{bx}{n}

\usepackage{amssymb}
\usepackage{amsthm}

\newtheorem{theorem}{Theorem}[section]

\theoremstyle{definition}

\theoremstyle{remark}
\newtheorem{remark}[theorem]{Remark}

\usepackage{hyperref}
\usepackage{url}

\title{Potential Matching Optimal Transport: Continuous Normalizing Flows for Exact $p$-Wasserstein Dynamics}

\author{
Lishuo Zhang$^{1}$ \quad Ruizhi Huang$^{1}$ \quad Yang Yu$^{1}$ \quad Lei Li$^{1,2}$\thanks{Corresponding author.} \\
$^{1}$School of Mathematical Sciences, Shanghai Jiao Tong University \\
$^{2}$Institute of Natural Sciences, MOE-LSC, Shanghai Jiao Tong University \\
Shanghai 200240, China \\
\texttt{ShawnLi9@sjtu.edu.cn} \quad
\texttt{hoperealizer@163.com} \\
\texttt{yuyang2357@sjtu.edu.cn} \quad
\texttt{leili2010@sjtu.edu.cn}
}

\iclrfinalcopy 
\begin{document}

\maketitle
\begin{abstract}
We introduce Potential Matching Optimal Transport (PMOT), a potential-flow
framework for general $p$-cost optimal transport with
$c_p(x,y)=\|x-y\|^p$. PMOT parameterizes the CNF velocity field with a scalar potential in the generalized Benamou--Brenier form for the chosen exponent $p$. It trains the potential gradient with a self-induced matching loss along straight bridges determined by the model's own endpoints, while allowing flexible terminal distribution
matching. Our main result establishes zero-loss exactness:
under the stated regularity, exact terminal matching, and uniqueness assumptions,
any zero-loss solution satisfies the generalized Benamou--Brenier optimality
system and recovers the corresponding $p$-optimal transport map and dynamics.
On synthetic benchmarks, PMOT learns $p$-specific maps that agree with the
corresponding $p$-matched OT references. It also remains competitive
as a likelihood-based density model on high-dimensional tabular data, and
an MMD-based color transformation experiment demonstrates flexible
sample-based terminal matching.
\end{abstract}

\section{Introduction}
Continuous Normalizing Flows (CNFs) learn invertible transformations between data distributions and tractable reference priors through time-dependent ODE dynamics \citep{chen2018neural, grathwohl2018ffjord}. They support exact likelihood evaluation via the continuous change-of-variables formula while simultaneously defining a transport trajectory between distributions. This makes CNFs a natural interface between likelihood-based generative modeling and optimal transport (OT).

A particularly influential line of work connects CNFs with OT by regularizing the learned dynamics toward low-cost transport paths. OT-Flow \citep{onken2021otflow}, for example, introduces a potential-flow parameterization together with a kinetic-energy regularizer, yielding a likelihood-based model whose dynamics are closely related to the Benamou--Brenier formulation of quadratic-cost optimal transport \citep{benamou2000computational}. However, the resulting geometry is naturally tied to the quadratic cost, or equivalently the $W_2$ geometry. Many OT problems are instead defined by more general costs
\[
    c_p(x,y)=\|x-y\|^p,
\]
where the exponent $p$ controls the induced transport geometry and, especially in asymmetric transport problems, can lead to different optimal couplings and trajectories \citep{villani2009optimal,peyre2019computational}. This motivates likelihood-based potential flows whose dynamics are aligned not only with the quadratic case, but also with general $p$-cost OT.

In this work, we introduce \emph{Potential Matching Optimal Transport} (PMOT), a potential-flow framework for general $p$-cost optimal transport with terminal distribution matching. PMOT preserves the main structural advantages of OT-Flow-style CNFs when instantiated with a KL/NLL terminal loss: the dynamics remain continuous, and likelihoods are evaluated through the CNF
change-of-variables formula. PMOT parameterizes the velocity field in the
generalized Benamou--Brenier form with conjugate exponent $q=p/(p-1)$, thereby
making the dynamics depend on the target $p$-cost geometry. Rather than using a
quadratic kinetic-energy penalty, PMOT introduces a self-induced
potential matching residual. For each data sample $x$, the current CNF produces
a terminal endpoint $F_\theta(x)$, and the residual matches the potential
gradient along the induced straight bridge to
\[
    -\|F_\theta(x)-x\|^{p-2}\bigl(F_\theta(x)-x\bigr).
\]
This condition is equivalent to matching the induced velocity to the
constant-speed bridge velocity. It requires no precomputed OT coupling and
instead imposes consistency on the model's own endpoint map. 

The central theoretical property of PMOT is a zero-loss exactness result.
Under the stated regularity, exact terminal matching, and uniqueness
assumptions, every zero-loss solution satisfies the generalized Benamou--Brenier
optimality system for the corresponding $p$-cost problem. Consequently, in
this idealized limit, the induced flow recovers the exact $p$-Wasserstein
dynamics and attains the optimal transport cost. In finite-sample neural
training, this exactness is subject to empirical sampling, model capacity,
numerical integration, and optimization error; our empirical claims are
therefore stated in terms of alignment with $p$-matched OT references rather
than unconditional exact recovery. Accordingly, we evaluate geometric
fidelity using post-hoc OT references computed with the same exponent as the
trained model.

We evaluate PMOT on synthetic transport, high-dimensional tabular density
estimation, and image color transformation. The experiments assess
$p$-matched geometric alignment, generation under mild objective weighting,
likelihood-based density modeling, and MMD-based sample transport.

Our contributions are summarized as follows:
\begin{itemize}
    \item We introduce PMOT, a potential-flow CNF framework for general
    $p$-cost optimal transport, trained through self-induced potential matching
    without external OT couplings or inner OT optimization.
    
    \item We prove that, under the stated regularity, exact terminal matching,
and uniqueness assumptions, every zero-loss PMOT solution recovers the
corresponding $p$-optimal transport map and dynamics.
    
    \item We evaluate each model against OT references computed with the matching
cost exponent and demonstrate $p$-specific geometric alignment, competitive
likelihood modeling, and flexible terminal matching through KL/NLL and MMD
objectives.
\end{itemize}

\section{Preliminaries}
\subsection{General $p$-Cost OT and Displacement Interpolation}

Let $\mu_0,\mu_1$ be probability measures on $\mathbb{R}^d$ with finite
$p$-th moments. For $p\geq1$, the $p$-cost optimal transport problem is
\[
W_p^p(\mu_0,\mu_1)
=
\inf_{\pi\in\Pi(\mu_0,\mu_1)}
\int \|x-y\|^p\,\mathrm d\pi(x,y).
\]
When an optimal Monge map $T_p$ exists, it satisfies
$T_{p\#}\mu_0=\mu_1$ and attains the same cost. Its displacement interpolation is
\[
X_t(x)=(1-t)x+tT_p(x),
\qquad
\rho_t=(X_t)_\#\mu_0,
\]
with velocity
\[
u_p(t,X_t(x))=T_p(x)-x.
\]
This path attains the dynamic $p$-cost:
\[
\int_0^1\int_{\mathbb{R}^d}\|u_p(t,z)\|^p\,
\mathrm d\rho_t(z)\,\mathrm dt
=
\int_{\mathbb{R}^d}\|T_p(x)-x\|^p\,\mathrm d\mu_0(x)
=
W_p^p(\mu_0,\mu_1).
\]

\subsection{Benamou--Brenier Optimality and Potential Velocities}

The same $p$-cost transport problem has a dynamic Benamou--Brenier
formulation\citep{santambrogio2015optimal}. The dynamic problem is
\[
    \inf_{\rho,v}
    \frac{1}{p}
    \int_0^1\int_{\mathbb{R}^d}
    \|v(t,x)\|^p\,d\rho_t(x)\,dt
\]
subject to the endpoint constraints $\rho_{t=0}=\mu_0$,
$\rho_{t=1}=\mu_1$ and the continuity equation
\[
    \partial_t\rho_t+\nabla_x\cdot(\rho_t v)=0 .
\]
In a smooth formal derivation, the scalar field $\Phi(t,x)$ appears as the
Lagrange multiplier for this continuity-equation constraint. The corresponding
Euler--Lagrange condition with respect to the velocity gives on the support of $\rho_t$ that
\[
    \|v\|^{p-2}v=-\nabla_x\Phi .
\]
Since $p>1$, the map $z\mapsto\|z\|^{p-2}z$ is invertible, and therefore the
optimal velocity has the potential form
\[
    v=-\|\nabla_x\Phi\|^{q-2}\nabla_x\Phi
    =
    -\nabla_\xi H_q(\nabla_x\Phi),
    \qquad
    H_q(\xi)=\frac{1}{q}\|\xi\|^q .
\]
Thus, the potential parameterization used by PMOT is not an arbitrary assumption:
it mirrors the velocity--multiplier relation in the smooth
Benamou--Brenier optimality system. The Euler--Lagrange condition with respect to $\rho$ then yields
$\partial_t\Phi+\nabla_x\Phi\cdot v+\frac{1}{p}\lVert v\rVert^p=0$ and thus on the support of $\rho_t$ that
\[
\partial_t\Phi-\frac{1}{q}\|\nabla\Phi\|^q=0.
\]
As we shall see, this corresponds to the fact that the velocity of a particle stays constant along its trajectory, a key feature of the OT. During neural training,
$\Phi_\theta$ is a model parameter; it coincides with the BB multiplier only in
the ideal zero-loss optimal regime, up to adding a function of time, which
does not change $\nabla_x\Phi_\theta$ or the induced velocity.

\subsection{Continuous Normalizing Flows}

A Continuous Normalizing Flow defines an invertible map through an ordinary differential equation
\[
    \frac{d z_t}{dt} = v_\theta(t,z_t), \qquad z_0=x.
\]
Let $F_\theta$ denote the time-one flow map, $F_\theta(x) = z_1$. If the terminal density is chosen as a tractable prior, such as $\mathcal{N}(0,I)$, the data likelihood is computed by the instantaneous change-of-variables formula:
\[
    \frac{d}{dt}\log \rho_t(z_t)
    =
    - \nabla \cdot v_\theta(t,z_t).
\]
Therefore,
\[
    \log \rho_0(x)
    =
    \log \rho_1(z_1)
    +
    \int_0^1
    \nabla \cdot v_\theta(t,z_t) \, dt,
\]
up to the sign convention determined by the integration direction. CNFs therefore provide likelihood evaluation while simultaneously defining a continuous transport trajectory.

\section{Potential Matching Optimal Transport}
\label{sec:pmot}

We introduce \emph{Potential Matching Optimal Transport} (PMOT), a variational
framework for general $p$-cost optimal transport. PMOT parameterizes the CNF
velocity field through a scalar potential and optimizes two terms: potential
matching along model-generated straight bridges and terminal distribution
matching. Under the assumptions stated below, every sufficiently regular
zero-loss solution recovers the unique $p$-optimal Monge map and its
corresponding Benamou--Brenier velocity field.

\subsection{Potential-Induced Flow}

Let $p>1$ and $q=p/(p-1)$. Following the Benamou--Brenier optimality relation
described above, we take a scalar potential
$\Phi:[0,1]\times\mathbb R^d\to\mathbb R$ as the primary variable and define
the induced velocity field
\begin{equation}
    v_\Phi(t,x)
    =
    -\lVert\nabla_x\Phi(t,x)\rVert^{q-2}\nabla_x\Phi(t,x).\label{eq:pmot_velocity}
\end{equation}
For $p=2$, \eqref{eq:pmot_velocity} reduces to the
familiar potential-flow relation
$v_\Phi(t,x)=-\nabla_x\Phi(t,x)$.

For each $x\in\mathbb R^d$, consider the initial-value problem
\begin{equation}\label{eq:pmot_flow}
\partial_t z(t)
=
v_\Phi\bigl(t,z(t)\bigr), \text{for almost every }t\in(0,1),
\qquad
z(0)=x.
\end{equation}
Throughout, we restrict attention to potentials $\Phi$ such that, for every $x\in\mathbb R^d$, problem~\eqref{eq:pmot_flow} admits a unique solution in $\mathrm{AC}([0,1];\mathbb R^d)$. We denote this solution by $z_\Phi(\cdot,x)$. And we further require that the associated solution map $z_\Phi:[0,1]\times\mathbb R^d\to\mathbb R^d$ is jointly Borel measurable.

The terminal flow map $F_\Phi:\mathbb R^d\to\mathbb R^d$ is defined by
\[
F_\Phi(x)=z_\Phi(1,x).
\]
The distribution induced at time $t$ is
\[
\rho_t^\Phi
=
\bigl(z_\Phi(t,\cdot)\bigr)_\#\mu_0.
\]
In particular, $\rho_0^\Phi=\mu_0,\rho_1^\Phi=(F_\Phi)_\#\mu_0$.

\subsection{The PMOT Variational Objective}

The potential matching loss is defined by
\[
\mathcal{L}_{\mathrm{PM}}(\Phi)
:=
\mathbb{E}_{x\sim\mu_0}
\left[
\int_0^1
\left\|
\nabla_x\Phi\!\left(
t,(1-t)x+tF_\Phi(x)
\right)
+
\left\|F_\Phi(x)-x\right\|^{p-2}
\left(F_\Phi(x)-x\right)
\right\|^2
\,\mathrm{d}t
\right].
\]

Let $\mathcal D$ be a nonnegative discrepancy between probability measures
satisfying $\mathcal D(P,Q)=0
\Leftrightarrow
P=Q$. We define the terminal distribution loss by
\[
\mathcal L_{\mathrm{term}}(\Phi)
:=
\mathcal D(\rho_1^\Phi,\mu_1).
\]
A principal example is the forward KL divergence, which gives $
\mathcal L_{\mathrm{term}}(\Phi)
=
\mathrm{KL}(\rho_1^\Phi\Vert\mu_1)$. 

We define the PMOT objective to be
\[
\mathcal L_{\mathrm{PMOT}}(\Phi)
:=
\lambda_{\mathrm{PM}}\mathcal L_{\mathrm{PM}}(\Phi)
+
\lambda_{\mathrm{term}}\mathcal L_{\mathrm{term}}(\Phi),
\]
where $\lambda_{\mathrm{PM}},\lambda_{\mathrm{term}}>0$ are fixed weights. The domain \(\operatorname{dom}(\mathcal L_{\mathrm{PMOT}})\) consists of all potentials \(\Phi\) satisfying the preceding flow well-posedness and joint Borel-measurability requirements for which both \(\mathcal L_{\mathrm{PM}}(\Phi)\) and \(\mathcal L_{\mathrm{term}}(\Phi)\) are well defined.

\begin{remark}[Velocity-space alternative]
Potential matching loss may be replaced by the velocity residual
\[
\left\|
v_\Phi\bigl(t,(1-t)x+tF_\Phi(x)\bigr)
-\bigl(F_\Phi(x)-x\bigr)
\right\|^2.
\]
The two formulations have the same zero set and coincide when $p=2$.
At finite loss, the nonlinear duality map changes their optimization
conditioning. This suggests that velocity matching may sometimes be preferable
for $p<2$ and potential matching for $p>2$, although neither uniformly
dominates the other. We use potential matching as the default formulation throughout this work.
\end{remark}

\subsection{Zero Loss Implies Optimal Transport}

Our main theorem is:

\begin{theorem}[Exact recovery and velocity-field uniqueness at zero loss]
\label{thm:pmot_exactness}
Let $p\in (1,\infty)$, let $\mu_0,\mu_1\in\mathcal P_p(\mathbb R^d)$, and write $\operatorname{dom}(\mathcal L_{\mathrm{PMOT}})$ for the domain of the PMOT objective. Assume that $\mu_0$ is absolutely continuous with respect to Lebesgue measure. Assume that the $p$-Benamou--Brenier problem admits a minimizer $(\rho^\star,v^\star)=(\rho^{\Phi^\star},v_{\Phi^\star})$ for some $\Phi^\star\in\operatorname{dom}(\mathcal L_{\mathrm{PMOT}})$. Then
\[
\min_{\Phi\in\operatorname{dom}(\mathcal L_{\mathrm{PMOT}})}
\mathcal L_{\mathrm{PMOT}}(\Phi)
=\mathcal L_{\mathrm{PMOT}}(\Phi^\star)=0.
\]
Furthermore, let $\Phi$ be a global minimizer. 
If $\Phi\in C^2((0,1)\times\mathbb R^d)$, and if
$\operatorname{supp}(\rho_t^\Phi)=\mathbb R^d$ for every $t\in(0,1)$, then $v_\Phi=v^\star$,  $\mathrm dt\,\rho_t^\Phi(\mathrm dx)$-almost everywhere, and the corresponding terminal map $F_\Phi$ is the unique optimal Monge map from $\mu_0$ to $\mu_1$ for the cost $\lVert x-y\rVert^p$.
\end{theorem}

Under the stated assumptions, this theorem provides the theoretical justification for PMOT: any sufficiently regular potential $\Phi$ attaining zero PMOT loss recovers the unique $p$-optimal Monge map and the corresponding Benamou--Brenier velocity field. Thus, in the ideal zero-loss setting, the PMOT objective exactly identifies the dynamic $p$-optimal transport solution. The proof is given in Appendix~\ref{app:pmot_proof_details}.

\begin{remark}
The second part of Theorem~\ref{thm:pmot_exactness} requires only interior regularity: \(\Phi\in C^2((0,1)\times\mathbb R^d)\), with no regularity assumption at \(t=0\) or \(t=1\). The \(C^2\) condition is imposed for simplicity; the proof also applies under weaker assumptions whenever the required weak derivatives, chain rules, and extension arguments remain valid.
Moreover, the nondegeneracy of the support of $\rho_t^{\Phi}$ for $t\in (0, 1)$ is essential for the minimizer to be the solution of OT. See Appendix \ref{app:counterexample_no_full_support} for counterexamples.
\end{remark}

\subsection{Neural Parameterization and Training}
\label{sec:neural_parameterization}

For computation, we restrict the potential to a neural family
$\{\Phi_\theta\}$ and use the scalar-valued ResNet potential architecture of
OT-Flow~\citep{onken2021otflow}. Let
$g_\theta(t,x)=\nabla_x\Phi_\theta(t,x)$. The implementation uses the stabilized
velocity
\begin{equation}
    v_{\theta,\epsilon}(t,x)
    =
    -\bigl(\|g_\theta(t,x)\|^2+\epsilon\bigr)^{(q-2)/2}
    g_\theta(t,x),
    \qquad
    \epsilon=10^{-12},
    \label{eq:pmot_parameterized_velocity}
\end{equation}
which approaches the population velocity in
\eqref{eq:pmot_velocity} as $\epsilon\to0$. We write $v_\theta$ for
$v_{\theta,\epsilon}$ below. The resulting ODE defines the flow
$z_\theta(t,x)$ and terminal map $F_\theta(x)=z_\theta(1,x)$.


Given a minibatch $\{x_0^{(i)}\}_{i=1}^B$ from $\mu_0$, a forward ODE solve
produces the model-induced endpoints
$x_1^{(i)}=F_\theta(x_0^{(i)})$. For independently sampled
$t_i\sim\mathcal U[0,1]$, define
\[
    \bar x_{t_i}^{(i)}
    =
    (1-t_i)x_0^{(i)}+t_i x_1^{(i)},
    \qquad
    u^{(i)}=x_1^{(i)}-x_0^{(i)} .
\]
Let
\[
    m_p(u)=\|u\|^{p-2}u,
    \qquad m_p(0)=0,
\]
denote the momentum dual to the $p$-power kinetic energy. The stochastic
self-induced potential-matching loss is
\begin{equation}
    \widehat{\mathcal L}_{\mathrm{PM}}(\theta)
    =
    \frac{1}{B}\sum_{i=1}^B
    \left\|
        \nabla_x\Phi_\theta
        \bigl(t_i,\bar x_{t_i}^{(i)}\bigr)
        +
        m_p\bigl(u^{(i)}\bigr)
    \right\|^2.
    \label{eq:pmot_empirical_pm}
\end{equation}

The model-induced endpoints remain attached to the computational graph, so
gradients propagate through $F_\theta$ in the potential-matching loss.

The population terminal objective remains
\begin{equation}
    \mathcal L_{\mathrm{term}}(\theta)
    =
    \mathcal D\bigl((F_\theta)_\#\mu_0,\mu_1\bigr).
    \label{eq:pmot_general_terminal_loss}
\end{equation}
For likelihood-based training, we use the forward-KL discrepancy and optimize
its equivalent CNF negative log-likelihood. Writing
\[
    \ell_\theta(x)
    =
    \int_0^1
    \nabla\cdot v_\theta\bigl(t,z_\theta(t,x)\bigr)\,\mathrm dt,
\]
the empirical NLL is
\[
    \widehat{\mathcal L}_{\mathrm{NLL}}(\theta)
    =
    -\frac{1}{B}\sum_{i=1}^B
    \left[
        \log\varphi_1\bigl(F_\theta(x_0^{(i)})\bigr)
        +
        \ell_\theta(x_0^{(i)})
    \right],
\]
where $\varphi_1$ is the terminal reference density. The NLL differs from the forward-KL objective by a
$\theta$-independent source-entropy constant. Therefore, the zero-loss
statement in Theorem~\ref{thm:pmot_exactness} refers to the population KL
discrepancy, not to the numerical NLL value reported in experiments. 

For sample-based terminal
matching, we instead use a minibatch estimate of squared MMD between
$\{F_\theta(x_0^{(i)})\}$ and independent samples from
$\mu_1$~\citep{gretton2012kernel}.

In either case, the implemented objective is
\begin{equation}
    \widehat{\mathcal J}_{\mathrm{PMOT}}(\theta)
    =
    \lambda_{\mathrm{PM}}
    \widehat{\mathcal L}_{\mathrm{PM}}(\theta)
    +
    \lambda_{\mathrm{term}}
    \widehat{\mathcal L}_{\mathrm{term}}(\theta).
    \label{eq:pmot_empirical_objective}
\end{equation}
No external OT pairs or HJB residual are used during training. We solve the
dynamics using a differentiable fixed-step RK4 scheme and backpropagate through
the ODE solve. Likelihood evaluation integrates data forward to the reference,
whereas generation integrates reference samples backward to the data space.

\section{Experiments}
\label{sec:experiments}

Our experiments are designed to evaluate three aspects of PMOT. First, we test whether PMOT endpoint maps agree with post-hoc OT references computed under the intended $p$-cost. Second, we examine whether the self-induced matching objective yields stable generation under mild likelihood weighting. Third, we evaluate whether PMOT remains competitive as a likelihood-based CNF on high-dimensional tabular density estimation. Across all experiments, Sinkhorn references are used only for evaluation and visualization, and never to construct training pairs.

\subsection{Toy Transport Geometry}
\label{sec:synthetic}

Two-dimensional synthetic distributions provide a controlled setting in which
post-hoc OT references can be estimated. For each exponent $p$, we
compute an entropically regularized Sinkhorn coupling with
$c_p(x,y)=\|x-y\|^p$~\citep{cuturi2013sinkhorn}. Its transport cost and
barycentric projection provide the reference cost and endpoint map,
respectively; the resulting coupling and its derived quantities are used only post hoc for
evaluation and visualization, never for training. We report the
model-induced $p$-cost, its ratio to the reference cost, and endpoint MSE to
the $p$-matched barycentric reference. Sinkhorn references are never used to
construct training pairs. Additional unconditional generation results for the
toy distributions are provided in Appendix~\ref{app:toy_generation}.

\begin{table}[t]
\centering
\caption{
OT-geometry evaluation on 8-Gaussians and Pinwheel. PMOT endpoints are compared
with post-hoc Sinkhorn barycentric references under the matching $p$-cost.
Sinkhorn is used only for evaluation, with $\varepsilon=0.02$ and 1000
iterations on 2000 samples. Ratios within a few percent of one indicate close
agreement with the reference.
}
\label{tab:ot_geometry_toy}
\vspace{0.5em}
\begin{tabular}{llcccc}
\toprule
Dataset & $p$ & PMOT $p$-cost & Sinkhorn $p$-cost & PMOT/Sinkhorn & Bary-MSE $\downarrow$ \\
\midrule
\multirow{3}{*}{8-Gaussians}
& 1.5 & 2.104 & 2.107 & 0.999 & 0.0166 \\
& 2.0 & 2.671 & 2.732 & 0.978 & 0.0144 \\
& 3.0 & 4.554 & 4.668 & 0.976 & 0.0139 \\
\midrule
\multirow{3}{*}{Pinwheel}
& 1.5 & 0.770 & 0.767 & 1.004 & 0.0195 \\
& 2.0 & 0.725 & 0.719 & 1.007 & 0.0208 \\
& 3.0 & 0.658 & 0.636 & 1.034 & 0.0219 \\
\bottomrule
\end{tabular}
\end{table}

Table~\ref{tab:ot_geometry_toy} gives the main quantitative geometry
evaluation. On 8-Gaussians, all cost ratios are within $2.4\%$ of one. On Pinwheel,
all ratios are within $3.4\%$ of one. Thus, when each model is evaluated against the matching
exponent, its endpoint map has both transport cost and barycentric discrepancy
close to the corresponding post-hoc OT reference.

To test whether the learned maps depend on the training exponent, we compare
each checkpoint trained with
$p_{\mathrm{train}}\in\{1.5,2,3\}$ against barycentric reference maps computed
with every $p_{\mathrm{ref}}\in\{1.5,2,3\}$. Each entry in
Figure~\ref{fig:pinwheel_p_swap_heatmap} reports the corresponding endpoint
MSE.

\begin{figure}[t]
    \centering
    \includegraphics[width=0.48\linewidth]{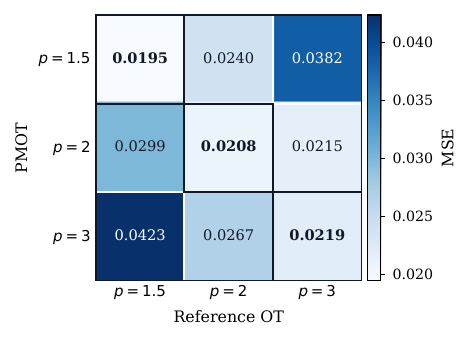}
    \caption{
    $p$-swap evaluation on Pinwheel.
Rows correspond to PMOT checkpoints trained with exponent
$p_{\mathrm{train}}$, and columns correspond to post-hoc Sinkhorn
barycentric reference maps computed with exponent $p_{\mathrm{ref}}$.
Each entry reports endpoint MSE for evaluation seed 0; lower is better.
Bold entries mark the minimum in each row.
Across five independently resampled evaluation batches with fixed checkpoints,
the mean off-diagonal endpoint MSE is $1.41\pm0.05$ times the mean diagonal
endpoint MSE.
    }
    \label{fig:pinwheel_p_swap_heatmap}
\end{figure}

Figure~\ref{fig:pinwheel_p_swap_heatmap} shows a diagonal preference: each
PMOT checkpoint is closest to the reference computed with its training
exponent. The separation is strongest between $p=1.5$ and $p=3$, while
neighboring exponents remain closer on this smooth toy distribution. This
supports $p$-specific behavior without requiring visually distinct
trajectories.

Moons provides a complementary qualitative example in a less symmetric
setting. Because entropic regularization can noticeably affect barycentric
references on this dataset, we use this experiment only to visualize the
learned coupling and continuous trajectories.

\begin{figure}[t]
    \centering
    \includegraphics[width=0.9\linewidth]{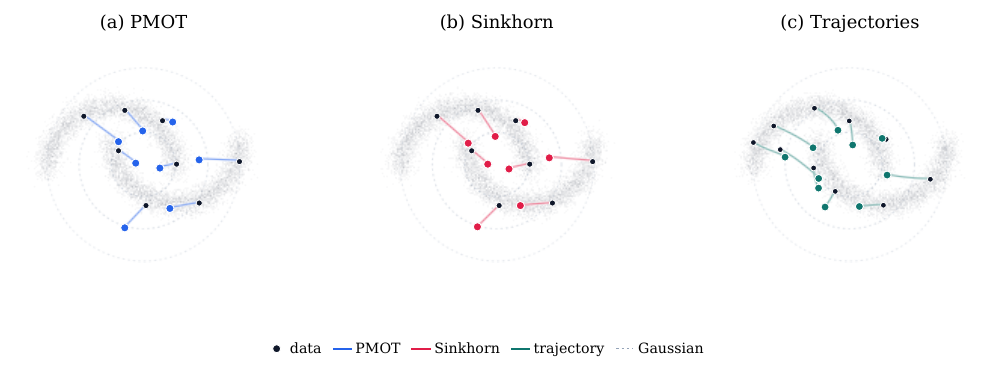}
    \caption{
    Learned transport geometry on two moons for $p=2$.
    (a) Coupling induced by the PMOT map.
    (b) Post-hoc Sinkhorn barycentric reference.
    (c) Continuous PMOT trajectories.
    The resulting coupling and its derived quantities are used only post hoc for
evaluation and visualization, never for training.
    }
    \label{fig:moons-ot-geometry-p2}
\end{figure}

Figure~\ref{fig:moons-ot-geometry-p2} compares the PMOT-induced map with a
post-hoc Sinkhorn reference and the continuous PMOT trajectories.

\subsection{Likelihood-Weight Sensitivity}
\label{sec:weight_sensitivity}

We next compare PMOT with OT-Flow on 8-Gaussians under equal training budgets. PMOT uses only self-induced potential matching and likelihood, with weights $(\lambda_{\mathrm{PM}},\lambda_{\mathrm{NLL}})=(1,1)$. OT-Flow uses its standard kinetic, likelihood, and HJB/R objective; we fix the kinetic and HJB/R coefficients to one and vary only the likelihood coefficient. This comparison tests whether good generation quality requires a strongly likelihood-dominated objective.

\begin{table}[t]
\centering
\caption{
Generation-quality comparison on 8-Gaussians. PMOT uses self-induced potential matching and likelihood weights $(\lambda_{\mathrm{PM}},\lambda_{\mathrm{NLL}})=(1,1)$ and has no HJB/R regularizer. For OT-Flow, $(\lambda_L,\lambda_C,\lambda_R)=(1,w_C,1)$ and only $w_C$ is varied. All models are trained for 10k iterations. Lower is better.
}
\label{tab:otflow_weight_sensitivity_8gaussians}
\vspace{0.5em}
\begin{tabular}{llccc}
\toprule
Method & Weights & Hist JS $\downarrow$ & Hist L1 $\downarrow$ & MMD $\downarrow$ \\
\midrule
PMOT & $(1,1)$
& \textbf{0.0272} & \textbf{0.2863} & \textbf{0.00063} \\
\midrule
\multirow{3}{*}[-0.3ex]{OT-Flow}
& $(1,1,1)$  & 0.1006 & 0.6951 & 0.00725 \\
& $(1,5,1)$  & 0.0316 & 0.3259 & 0.00136 \\
& $(1,10,1)$ & 0.0303 & 0.3118 & 0.00108 \\
\bottomrule
\end{tabular}
\end{table}

\begin{figure}[t]
    \centering
    \includegraphics[width=0.72\linewidth]{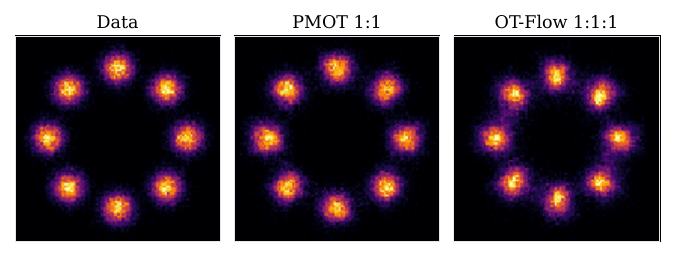}
    \caption{
Generation quality on 8-Gaussians under a limited training budget of 1000 iterations.
PMOT uses the mild balanced weights $(\lambda_{\mathrm{PM}},\lambda_{\mathrm{NLL}})=(1,1)$, while OT-Flow uses $(\lambda_L,\lambda_C,\lambda_R)=(1,1,1)$.
Under the same training budget and mild weighting, PMOT recovers the eight modes more clearly, while OT-Flow exhibits stronger inter-mode artifacts.
}
    \label{fig:pmot_otflow_budget_8gaussians}
\end{figure}

Figure~\ref{fig:pmot_otflow_budget_8gaussians} shows the same effect visually under a shorter 1000-iteration budget: PMOT already recovers the eight modes clearly with balanced weights, whereas OT-Flow with $(1,1,1)$ exhibits stronger inter-mode artifacts. Table~\ref{tab:otflow_weight_sensitivity_8gaussians} quantifies the trend after 10k iterations. OT-Flow improves substantially as the likelihood coefficient is increased, whereas PMOT obtains the best histogram and MMD scores with the balanced $(1,1)$ objective. This supports the view that the self-induced matching term provides useful transport structure without adding an HJB/R regularizer.

\subsection{High-Dimensional Density Estimation}
\label{sec:highdim_density}

\begin{table}[t]
\centering
\caption{
High-dimensional tabular density estimation. For PMOT, weights are reported as
$(\lambda_{\mathrm{PM}},\lambda_{\mathrm{NLL}})$; for OT-Flow, weights are reported as
$(\lambda_L,\lambda_C,\lambda_R)$. PMOT uses self-induced potential matching and likelihood without an HJB/R regularizer. MMD compares inverse-generated samples with held-out data. Lower is better.
}
\label{tab:highdim_density}
\vspace{0.5em}
\begin{tabular}{lllcc}
\toprule
Dataset & Method & Weights & NLL $\downarrow$ & MMD $\downarrow$ \\
\midrule
\multirow{2}{*}[-0.3ex]{MiniBooNE}
& PMOT & $(1,5)$ & 10.454 & $6.67{\times}10^{-4}$ \\
& OT-Flow & $(1,100,15)$ & 10.634 & $6.67{\times}10^{-4}$ \\
\midrule
\multirow{2}{*}[-0.3ex]{POWER}
& PMOT & $(1,5)$ & -0.383 & $2.91{\times}10^{-4}$ \\
& OT-Flow & $(1,500,5)$ & -0.312 & $3.57{\times}10^{-4}$ \\
\midrule
\multirow{2}{*}[-0.3ex]{HEPMASS}
& PMOT & $(1,5)$ & 16.922 & $5.00{\times}10^{-4}$ \\
& OT-Flow & $(1,500,40)$ & 17.351 & $5.00{\times}10^{-4}$ \\
\bottomrule
\end{tabular}
\end{table}

Table~\ref{tab:highdim_density} compares PMOT and OT-Flow on MiniBooNE, POWER, and HEPMASS \citep{papamakarios2017masked,grathwohl2018ffjord,onken2021otflow}. PMOT uses the same mild $(1,5)$ weighting across all three datasets and does not include an HJB/R regularizer. OT-Flow uses dataset-specific likelihood and HJB/R coefficients. In these runs, PMOT obtains lower NLL on all three datasets and comparable generation MMD, indicating that the simpler objective remains competitive in high-dimensional tabular density estimation.

\paragraph{Image color transformation.}
To illustrate sample-based terminal matching, we apply MMD-based PMOT to image
color transformation. Figure~\ref{fig:color_transform_transfer_montage} shows
the transformation along the learned flow; experimental details and additional
visualizations are provided in Appendix~\ref{app:color_transform}.

\begin{figure}[t]
    \centering
    \includegraphics[width=\linewidth]
    {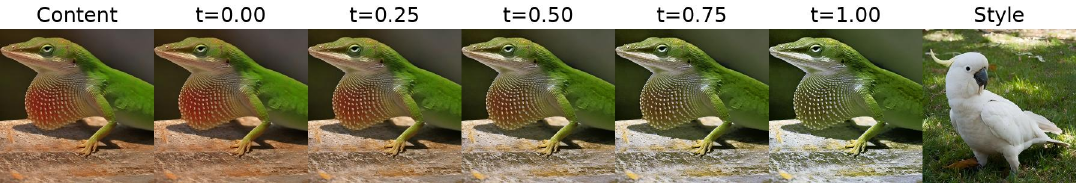}
    \caption{
    MMD-based image color transformation along the learned PMOT flow.
    The leftmost panel is the source image and the rightmost panel is the target
    style image.
    }
    \label{fig:color_transform_transfer_montage}
\end{figure}


\section{Related Work}
\label{sec:related_work}

\paragraph{CNFs and optimal-transport dynamics.}
Continuous Normalizing Flows (CNFs) define invertible transformations through neural ordinary differential equations and evaluate likelihoods using the instantaneous change-of-variables formula \citep{chen2018neural,grathwohl2018ffjord}. This makes CNFs natural candidates for learning continuous transport maps, but likelihood maximization alone does not identify a unique path between the data distribution and the reference prior. OT-Flow addresses this ambiguity by combining a potential-flow CNF with kinetic-energy and HJB regularization, yielding dynamics closely related to the Benamou--Brenier formulation of quadratic-cost optimal transport \citep{benamou2000computational,onken2021otflow, jing2025convergence}. Recent work has also used machine learning to approximate geodesic dynamics under related transport geometries, such as the spherical
Wasserstein--Fisher--Rao metric \citep{jing2024machine}. PMOT likewise learns
continuous transport dynamics, but focuses on likelihood-based potential flows
for general $p$-cost Wasserstein geometry \citep{villani2009optimal,peyre2019computational}. It parameterizes the CNF velocity in the generalized Benamou--Brenier form for
the chosen exponent $p$ and replaces the quadratic kinetic-energy regularizer
with a self-induced potential-matching objective .

\paragraph{Flow matching and neural OT with general costs.}
Flow matching trains continuous-time generative models by regressing velocity fields along prescribed probability paths, avoiding the need to backpropagate through ODE solves during training \citep{lipman2023flowmatching}. OT-coupled flow-matching methods use minibatch OT couplings to define training pairs or conditional paths \citep{tong2023improving}. PMOT instead uses a potential matching residual whose endpoint pair is induced by the current CNF map, so no external OT coupling is used during training. Neural OT methods also learn maps or couplings under general cost functions, including costs beyond the quadratic case \citep{korotin2021neuralot,asadulaev2024neuralgeneralcost}. PMOT is complementary to this literature: rather than learning only a static map or using minibatch OT pairs as supervision, it learns continuous potential-flow dynamics while preserving CNF likelihood evaluation and allowing flexible terminal matching such as KL/NLL or MMD.

\section{Conclusion}
\label{sec:conclusion}

We presented Potential Matching Optimal Transport (PMOT), a potential-flow CNF
framework for general $p$-cost optimal transport. PMOT parameterizes the CNF
velocity field with a scalar potential in the generalized Benamou--Brenier
form and trains the potential gradient using a self-induced matching loss on
straight bridges determined by the model endpoints. It requires no external
transport pairs or inner OT optimization and supports both likelihood-based
training through a KL/NLL terminal objective and sample-based terminal matching
through MMD.

Our main result establishes exactness at zero population PMOT loss: under the
stated realizability, regularity, and full-support assumptions, every global
minimizer recovers the unique $p$-optimal Monge map and its
Benamou--Brenier dynamics. On synthetic benchmarks, the learned maps agree
with $p$-matched OT references and are closest to the references computed
using their respective training exponents. PMOT also remains competitive as a
likelihood-based density model on high-dimensional tabular data, while an
additional MMD-based color transformation experiment demonstrates
sample-based terminal matching. As with deterministic CNFs more broadly, PMOT
represents Monge-type maps, and finite-sample training, numerical integration,
and optimization introduce practical approximation errors.

\bibliography{iclr2026_conference}

@article{benamou2000computational,
  title={A Computational Fluid Mechanics Solution to the Monge-Kantorovich Mass Transfer Problem},
  author={Benamou, Jean-David and Brenier, Yann},
  journal={Numerische Mathematik},
  volume={84},
  number={3},
  pages={375--393},
  year={2000},
  publisher={Springer},
  doi={10.1007/s002110050002}
}

@book{villani2009optimal,
  title={Optimal transport: old and new},
  author={Villani, C{\'e}dric and others},
  volume={338},
  year={2009},
  publisher={Springer}
}

@article{peyre2019computational,
  title={Computational Optimal Transport},
  author={Peyr{\'e}, Gabriel and Cuturi, Marco},
  journal={Foundations and Trends in Machine Learning},
  volume={11},
  number={5--6},
  pages={355--607},
  year={2019},
  publisher={Now Publishers},
  doi={10.1561/2200000073}
}

@book{santambrogio2015optimal,
  title={Optimal Transport for Applied Mathematicians: Calculus of Variations, PDEs, and Modeling},
  author={Santambrogio, Filippo},
  series={Progress in Nonlinear Differential Equations and Their Applications},
  volume={87},
  publisher={Birkh{\"a}user Cham},
  year={2015},
  doi={10.1007/978-3-319-20828-2}
}

@article{chen2018neural,
  title={Neural ordinary differential equations},
  author={Chen, Ricky TQ and Rubanova, Yulia and Bettencourt, Jesse and Duvenaud, David K},
  journal={Advances in neural information processing systems},
  volume={31},
  year={2018}
}

@article{grathwohl2018ffjord,
  title={Ffjord: Free-form continuous dynamics for scalable reversible generative models},
  author={Grathwohl, Will and Chen, Ricky TQ and Bettencourt, Jesse and Sutskever, Ilya and Duvenaud, David},
  journal={arXiv preprint arXiv:1810.01367},
  year={2018}
}

@inproceedings{onken2021otflow,
  title={{OT-Flow}: Fast and Accurate Continuous Normalizing Flows via Optimal Transport},
  author={Onken, Derek and Wu Fung, Samy and Li, Xingjian and Ruthotto, Lars},
  booktitle={Proceedings of the AAAI Conference on Artificial Intelligence},
  volume={35},
  pages={9223--9232},
  year={2021},
  doi={10.1609/aaai.v35i10.17113},
  url={https://ojs.aaai.org/index.php/AAAI/article/view/17113}
}

@inproceedings{lipman2023flowmatching,
  title={Flow Matching for Generative Modeling},
  author={Lipman, Yaron and Chen, Ricky T. Q. and Ben-Hamu, Heli and Nickel, Maximilian and Le, Matt},
  booktitle={International Conference on Learning Representations},
  year={2023},
  url={https://openreview.net/forum?id=PqvMRDCJT9t}
}

@article{gretton2012kernel,
  title={A Kernel Two-Sample Test},
  author={Gretton, Arthur and Borgwardt, Karsten M. and Rasch, Malte J. and Sch{\"o}lkopf, Bernhard and Smola, Alexander},
  journal={Journal of Machine Learning Research},
  volume={13},
  number={25},
  pages={723--773},
  year={2012},
  url={https://www.jmlr.org/papers/v13/gretton12a.html}
}

@inproceedings{cuturi2013sinkhorn,
  title={Sinkhorn Distances: Lightspeed Computation of Optimal Transport},
  author={Cuturi, Marco},
  booktitle={Advances in Neural Information Processing Systems},
  year={2013}
}

@inproceedings{papamakarios2017masked,
  title={Masked Autoregressive Flow for Density Estimation},
  author={Papamakarios, George and Pavlakou, Theo and Murray, Iain},
  booktitle={Advances in Neural Information Processing Systems},
  year={2017}
}

@inproceedings{korotin2021neuralot,
  title={Neural Optimal Transport},
  author={Korotin, Alexander and Li, Lingxiao and Genevay, Aude and Solomon, Justin M. and Filippov, Alexander and Burnaev, Evgeny},
  booktitle={International Conference on Learning Representations},
  year={2023},
  url={https://openreview.net/forum?id=d8CBRlWNkqH}
}

@inproceedings{asadulaev2024neuralgeneralcost,
  title={Neural Optimal Transport with General Cost Functionals},
  author={Asadulaev, Arip and Korotin, Alexander and Egiazarian, Vage and Mokrov, Petr and Burnaev, Evgeny},
  booktitle={International Conference on Learning Representations},
  year={2024},
  url={https://openreview.net/forum?id=gIiz7tBtYZ}
}

@article{tong2023improving,
  title={Improving and generalizing flow-based generative models with minibatch optimal transport},
  author={Tong, Alexander and Fatras, Kilian and Malkin, Nikolay and Huguet, Guillaume and Zhang, Yanlei and Rector-Brooks, Jarrid and Wolf, Guy and Bengio, Yoshua},
  journal={arXiv preprint arXiv:2302.00482},
  year={2023}
}

@INPROCEEDINGS{deng2009imagenet,
  author={Deng, Jia and Dong, Wei and Socher, Richard and Li, Li-Jia and Kai Li and Li Fei-Fei},
  booktitle={2009 IEEE Conference on Computer Vision and Pattern Recognition}, 
  title={ImageNet: A large-scale hierarchical image database}, 
  year={2009},
  volume={},
  number={},
  pages={248-255},
  doi={10.1109/CVPR.2009.5206848}}

@article{russakovsky2015imagenet,
Author = {Olga Russakovsky and Jia Deng and Hao Su and Jonathan Krause and Sanjeev Satheesh and Sean Ma and Zhiheng Huang and Andrej Karpathy and Aditya Khosla and Michael Bernstein and Alexander C. Berg and Li Fei-Fei},
Title = {{ImageNet Large Scale Visual Recognition Challenge}},
Year = {2015},
journal   = {International Journal of Computer Vision (IJCV)},
doi = {10.1007/s11263-015-0816-y},
volume={115},
number={3},
pages={211-252}
}

@article{jing2024machine,
  title={A machine learning framework for geodesics under spherical Wasserstein--Fisher--Rao metric and its application for weighted sample generation},
  author={Jing, Yang and Chen, Jiaheng and Li, Lei and Lu, Jianfeng},
  journal={Journal of Scientific Computing},
  volume={98},
  number={1},
  pages={5},
  year={2024},
  publisher={Springer}
}

@article{jing2025convergence,
  title={Convergence Analysis of OT-Flow for Sample Generation},
  author={Jing, Yang and Li, Lei},
  journal={Numerical Mathematics: Theory, Methods and Applications},
  volume={18},
  number={2},
  pages={325--352},
  year={2025}
}
\bibliographystyle{iclr2026_conference}

\appendix

\section{Details for Theorem~\ref{thm:pmot_exactness}}
\label{app:pmot_proof_details}

\subsection{Proof of Theorem~\ref{thm:pmot_exactness}}
\begin{proof}
(1) By assumption, $(\rho^\star,v^\star)=(\rho^{\Phi^\star},v_{\Phi^\star})$ is a solution of the $p$-Benamou--Brenier problem for some $\Phi^\star\in\operatorname{dom}(\mathcal L_{\mathrm{PMOT}})$.

The endpoint condition of the Benamou--Brenier problem and the definition of the induced flow give
\[
(F_{\Phi^\star})_\sharp\mu_0
=\rho_1^{\Phi^\star}
=\rho_1^\star
=\mu_1.
\]
Since $(\rho^{\Phi^\star},v_{\Phi^\star})$ minimizes the Benamou--Brenier action, the pushforward formula and the characteristic equation yield
\[
\begin{aligned}
\frac{1}{p}W_p^p(\mu_0,\mu_1)
&=\frac{1}{p}\int_0^1\int_{\mathbb R^d}
\lVert v_{\Phi^\star}(t,y)\rVert^p\,
\rho_t^{\Phi^\star}(\mathrm dy)\,\mathrm dt\\
&=\frac{1}{p}\int_{\mathbb R^d}\int_0^1
\left\lVert\partial_tz_{\Phi^\star}(t,x)\right\rVert^p
\,\mathrm dt\,\mu_0(\mathrm dx)\\
&\geq\frac{1}{p}\int_{\mathbb R^d}
\left\lVert F_{\Phi^\star}(x)-x\right\rVert^p
\,\mu_0(\mathrm dx)\\
&\geq\frac{1}{p}W_p^p(\mu_0,\mu_1).
\end{aligned}
\]
The first inequality is Jensen's inequality applied to $\mu_0$-almost every trajectory of finite $p$-energy, using
\[
F_{\Phi^\star}(x)-x
=\int_0^1\partial_tz_{\Phi^\star}(t,x)\,\mathrm dt,
\]
and the second follows because $(\operatorname{Id},F_{\Phi^\star})_\sharp\mu_0\in\Pi(\mu_0,\mu_1)$. Since the leftmost and rightmost quantities in the preceding chain coincide, equality holds throughout. In particular, $(\operatorname{Id},F_{\Phi^\star})_\sharp\mu_0$ is an optimal coupling. Since $\mu_0$ is absolutely continuous and $p>1$, the optimal coupling is induced by the unique optimal Monge map $T_p$, and therefore
\[
F_{\Phi^\star}=T_p
\qquad
\mu_0\text{-almost everywhere}.
\]
Moreover, the nonnegative gap in Jensen's inequality has zero integral with respect to $\mu_0$. Thus equality in Jensen's inequality holds for $\mu_0$-almost every $x$. Since $w\mapsto\lVert w\rVert^p$ is strictly convex for $p>1$, this implies
\[
\partial_tz_{\Phi^\star}(t,x)
=F_{\Phi^\star}(x)-x
\qquad
\text{for almost every }t\in(0,1)
\]
for $\mu_0$-almost every $x$. Since $z_{\Phi^\star}(\cdot,x)$ is absolutely continuous and $z_{\Phi^\star}(0,x)=x$, integration in time gives
\[
z_{\Phi^\star}(t,x)
=(1-t)x+tF_{\Phi^\star}(x)
=(1-t)x+tT_p(x)
\]
for every $t\in[0,1]$ and for $\mu_0$-almost every $x$. Combining this identity with the characteristic equation gives
\[
v_{\Phi^\star}\bigl(t,(1-t)x+tF_{\Phi^\star}(x)\bigr)
=F_{\Phi^\star}(x)-x
\qquad
\mathrm dt\otimes\mu_0\text{-almost everywhere}.
\]
For $r>1$, let $A_r(\xi):=\lVert\xi\rVert^{r-2}\xi$, with $A_r(0):=0$. The duality maps $A_p$ and $A_q$ are odd and inverse to each other. Since $v_{\Phi^\star}=-A_q(\nabla_x\Phi^\star)$, the preceding velocity identity is equivalent to
\[
\nabla_x\Phi^\star\bigl(t,(1-t)x+tF_{\Phi^\star}(x)\bigr)
=-\lVert F_{\Phi^\star}(x)-x\rVert^{p-2}
\bigl(F_{\Phi^\star}(x)-x\bigr)
\qquad
\mathrm dt\otimes\mu_0\text{-almost everywhere}.
\]
Thus the potential-matching residual vanishes and $\mathcal L_{\mathrm{PM}}(\Phi^\star)=0$. The endpoint identity above also gives $\mathcal L_{\mathrm{term}}(\Phi^\star)=0$. Since the two terms in the PMOT objective are nonnegative and have positive weights,
\[
\min_{\Phi\in\operatorname{dom}(\mathcal L_{\mathrm{PMOT}})}
\mathcal L_{\mathrm{PMOT}}(\Phi)
=
\mathcal L_{\mathrm{PMOT}}(\Phi^\star)
=
0.
\]

(2) Now let $\Phi$ be a global minimizer satisfying the regularity and full-support assumptions of the theorem. By part~(1), the minimum value of the PMOT objective is zero. Hence
\[
\mathcal L_{\mathrm{PM}}(\Phi)=\mathcal L_{\mathrm{term}}(\Phi)=0,
\]
and in particular, $(F_\Phi)_\#\mu_0=\mu_1$.


Define
\[
S_t(x):=(1-t)x+tF_\Phi(x)
\]
and the potential-matching residual
\[
R(t,x):=
\nabla_x\Phi\bigl(t,S_t(x)\bigr)
+\lVert F_\Phi(x)-x\rVert^{p-2}\bigl(F_\Phi(x)-x\bigr),
\qquad (t,x)\in(0,1)\times\mathbb R^d.
\]
Since the solution map \(z_\Phi\) is jointly Borel measurable, its terminal-time slice \(F_\Phi=z_\Phi(1,\cdot)\) is Borel measurable. Consequently, the map \((t,x)\mapsto S_t(x)\) is jointly Borel measurable. Moreover, \(\nabla_x\Phi\) is continuous on \((0,1)\times\mathbb R^d\), and the map \(\xi\mapsto\lVert\xi\rVert^{p-2}\xi\), with value \(0\) at \(\xi=0\), is continuous for \(p>1\). It follows that \(R\) is jointly Borel measurable.

The function \((t,x)\mapsto\lVert R(t,x)\rVert^2\) is therefore nonnegative and jointly Borel measurable. By Tonelli's theorem, the function
\[
h(x):=\int_0^1\lVert R(t,x)\rVert^2\,\mathrm dt
\]
is measurable. Since \(\mathcal L_{\mathrm{PM}}(\Phi)=0\),
\[
0
=
\int_{\mathbb R^d}h(x)\,\mu_0(\mathrm dx).
\]
Because \(h\geq0\), it follows that \(h=0\) for \(\mu_0\)-almost every \(x\). Hence the measurable set
\[
E:=\{x\in\mathbb R^d:h(x)=0\}
\]
satisfies \(\mu_0(E)=1\). For every \(x\in E\), we have
\[
\int_0^1\lVert R(t,x)\rVert^2\,\mathrm dt=0,
\]
and therefore
\[
R(t,x)=0
\qquad
\text{for almost every }t\in(0,1).
\]
For each fixed \(x\in E\), the map \(t\mapsto R(t,x)\) is continuous on \((0,1)\). Since it vanishes for almost every \(t\in(0,1)\), it must vanish for every \(t\in(0,1)\).

Thus, for every \(x\in E\) and \(t\in(0,1)\),
\[
\nabla_x\Phi\bigl(t,S_t(x)\bigr)
=-A_p\bigl(F_\Phi(x)-x\bigr).
\]
Using \(v_\Phi=-A_q(\nabla_x\Phi)\), together with the oddness and inverse relation of the duality maps defined in part~(1), gives
\[
v_\Phi\bigl(t,S_t(x)\bigr)
=-A_q\bigl(-A_p(F_\Phi(x)-x)\bigr)
=F_\Phi(x)-x.
\]

Since
\[
\partial_tS_t(x)=F_\Phi(x)-x,
\]
it follows that, for every $x\in E$,
\[
\partial_tS_t(x)=v_\Phi\bigl(t,S_t(x)\bigr),
\qquad t\in(0,1),
\qquad S_0(x)=x.
\]
Since $S_\cdot(x)$ is affine in time, it is absolutely continuous on $[0,1]$ and satisfies the characteristic equation on $(0,1)$. Hence $S_\cdot(x)$ solves problem~\eqref{eq:pmot_flow}, and uniqueness gives
\[
z_\Phi(t,x)=S_t(x)
\qquad
\text{for every }t\in[0,1]\text{ and every }x\in E.
\]
In particular,
\[
v_\Phi\bigl(t,z_\Phi(t,x)\bigr)=F_\Phi(x)-x
\qquad
\text{for every }t\in(0,1)\text{ and every }x\in E.
\]
Thus, for $\mu_0$-almost every initial point, the induced trajectory is a constant-speed straight line and the velocity is constant along it.

It remains to show that the pair $(\rho_t^\Phi,v_\Phi)$ induced by these straight-line trajectories is a minimizer of the $p$-Benamou--Brenier problem. We first prove that $F_\Phi$ is the optimal map for the static $p$-cost transport problem and then verify that the induced straight-line flow attains the minimum Benamou--Brenier action. Set
\[
Q:=(0,1)\times\mathbb R^d,
\qquad
q=\frac{p}{p-1},
\qquad
H_q(\xi)=\frac{1}{q}\lVert\xi\rVert^q,
\qquad
A_q(\xi)=\nabla H_q(\xi)=\lVert\xi\rVert^{q-2}\xi,
\]
and write $g=\nabla_x\Phi$, so that $v_\Phi=-A_q(g)$. Define
\[
G(t,x):=\partial_t\Phi(t,x)-H_q(g(t,x)),
\qquad (t,x)\in Q.
\]
We claim that $\nabla_xG=0$ on $Q$. Since $\Phi\in C^2(Q)$, the field $g=\nabla_x\Phi$ is $C^1$ on $Q$. Moreover, for every $q>1$, the map $A_q(\xi)=\lVert\xi\rVert^{q-2}\xi$ is $C^1$ on $\mathbb R^d\setminus\{0\}$, with
\[
DA_q(\xi)=\lVert\xi\rVert^{q-2}I+(q-2)\lVert\xi\rVert^{q-4}\xi\otimes\xi.
\]
Since $g$ is continuous, the set
\[
U:=\{(t,x)\in Q:g(t,x)\neq0\}
\]
is open. Therefore, $v_\Phi=-A_q\circ g$ is $C^1$ on $U$ by the chain rule. For every $x\in E$, the velocity is constant along the corresponding characteristic:
\[
v_\Phi\bigl(t,z_\Phi(t,x)\bigr)=F_\Phi(x)-x,
\qquad t\in(0,1).
\]
By the straight-line identity, $z_\Phi(\cdot,x)=S_\cdot(x)$ is affine and therefore $C^1$ on $(0,1)$. Since $U$ is open and $z_\Phi(\cdot,x)$ is continuous, whenever $x\in E$ and $(t,z_\Phi(t,x))\in U$, the trajectory remains in $U$ on a neighborhood of $t$. The $C^1$ chain rule therefore applies on this neighborhood, and differentiating the preceding identity with respect to $t$ and using $\partial_tz_\Phi(t,x)=v_\Phi(t,z_\Phi(t,x))$ gives
\[
\begin{aligned}
0
&=\frac{\mathrm d}{\mathrm dt}v_\Phi\bigl(t,z_\Phi(t,x)\bigr)\\
&=\partial_tv_\Phi\bigl(t,z_\Phi(t,x)\bigr)+D_xv_\Phi\bigl(t,z_\Phi(t,x)\bigr)\,\partial_tz_\Phi(t,x)\\
&=\bigl[\partial_tv_\Phi+D_xv_\Phi\,v_\Phi\bigr]\bigl(t,z_\Phi(t,x)\bigr).
\end{aligned}
\]
Define the space--time measure as the pushforward
\[
\nu
:=
\bigl((t,x)\mapsto(t,z_\Phi(t,x))\bigr)_\#
\bigl(\mathrm dt\otimes\mu_0\bigr).
\]
The joint Borel measurability of $z_\Phi$ ensures that this measure is well defined. Since $\rho_t^\Phi=(z_\Phi(t,\cdot))_\#\mu_0$, it can equivalently be written as
\[
\nu(\mathrm dt,\mathrm dy)=\mathrm dt\,\rho_t^\Phi(\mathrm dy).
\]
Since $\mu_0(E)=1$, the preceding identity implies
\[
\partial_tv_\Phi+D_xv_\Phi\,v_\Phi=0
\qquad \nu\text{-almost everywhere on }U.
\]
The assumption $\operatorname{supp}(\rho_t^\Phi)=\mathbb R^d$ for every $t\in(0,1)$ implies that $\nu$ has full support on $Q$. Indeed, every nonempty open subset of $Q$ contains a product $I\times B$, where $I\subset(0,1)$ and $B\subset\mathbb R^d$ are nonempty and open, and
\[
\nu(I\times B)=\int_I\rho_t^\Phi(B)\,\mathrm dt>0.
\]
Since the map
\[
(t,x)\longmapsto
\partial_tv_\Phi(t,x)+D_xv_\Phi(t,x)v_\Phi(t,x)
\]
is continuous on $U$, it follows that
\[
(\partial_t+v_\Phi\cdot\nabla_x)v_\Phi
=\partial_tv_\Phi+D_xv_\Phi\,v_\Phi
=0
\qquad\text{on }U.
\]
To relate the material derivative of $v_\Phi$ to $\nabla_xG$, note that $D_xg=D_x^2\Phi$, $DA_q=D^2H_q$, and $(v_\Phi\cdot\nabla_x)v_\Phi=D_xv_\Phi\,v_\Phi$. Hence, the chain rule yields
\[
\begin{aligned}
(\partial_t+v_\Phi\cdot\nabla_x)v_\Phi
&=\partial_tv_\Phi+D_xv_\Phi\,v_\Phi\\
&=-DA_q(g)\,\partial_tg-DA_q(g)D_xg\,v_\Phi\\
&=-D^2H_q(g)\bigl(\partial_tg+D_x^2\Phi\,v_\Phi\bigr).
\end{aligned}
\]
On the other hand, using the symmetry of $D_x^2\Phi$ and $v_\Phi=-A_q(g)$,
\[
\begin{aligned}
\nabla_xG
&=\partial_tg-\nabla_x\bigl(H_q(g)\bigr)\\
&=\partial_tg-(D_xg)^\top A_q(g)\\
&=\partial_tg-D_x^2\Phi\,A_q(g)\\
&=\partial_tg+D_x^2\Phi\,v_\Phi.
\end{aligned}
\]
Combining the preceding two identities yields
\[
0=(\partial_t+v_\Phi\cdot\nabla_x)v_\Phi=-D^2H_q(g)\nabla_xG
\qquad\text{on }U.
\]
For $g\neq0$,
\[
D^2H_q(g)=\lVert g\rVert^{q-2}I+(q-2)\lVert g\rVert^{q-4}g\otimes g
\]
is positive definite. Indeed, for any $\eta\in\mathbb R^d$, define its components parallel and orthogonal to $g$ by
\[
\eta_\parallel:=\frac{\eta\cdot g}{\lVert g\rVert^2}g,
\qquad
\eta_\perp:=\eta-\eta_\parallel.
\]
Then $\eta=\eta_\perp+\eta_\parallel$, $\eta_\perp\cdot g=0$, and $(g\cdot\eta)^2=\lVert g\rVert^2\lVert\eta_\parallel\rVert^2$. Hence
\[
\begin{aligned}
\eta^\top D^2H_q(g)\eta
&=\lVert g\rVert^{q-2}\lVert\eta\rVert^2+(q-2)\lVert g\rVert^{q-4}(g\cdot\eta)^2\\
&=\lVert g\rVert^{q-2}\bigl(\lVert\eta_\perp\rVert^2+\lVert\eta_\parallel\rVert^2+(q-2)\lVert\eta_\parallel\rVert^2\bigr)\\
&=\lVert g\rVert^{q-2}\bigl(\lVert\eta_\perp\rVert^2+(q-1)\lVert\eta_\parallel\rVert^2\bigr)>0
\end{aligned}
\]
for every $\eta\neq0$, since $g\neq0$ and $q>1$. Equivalently, the tangential eigenvalue is $\lVert g\rVert^{q-2}$ and the radial eigenvalue is $(q-1)\lVert g\rVert^{q-2}$. Therefore $\nabla_xG=0$ on $U$.

It remains to consider the zero set
\[
Z:=\{(t,x)\in Q:g(t,x)=0\}.
\]
At every interior point of $Z$ relative to $Q$, the field $g$ vanishes in a neighborhood, so $\partial_tg=0$ there and hence $\nabla_xG=0$. Every other point of $Z$ is a limit point of $U$, and the continuity on $Q$ of
\[
\nabla_xG=\partial_tg-D_x^2\Phi\,A_q(g)
\]
again yields $\nabla_xG=0$. Thus $\nabla_xG=0$ on $Z$ as well. Together with the conclusion on $U$, this yields
\[
\nabla_xG=0
\qquad\text{on }Q.
\]

Since $\mathbb R^d$ is connected, there is a continuous function $a:(0,1)\to\mathbb R$ such that
\[
G(t,x)=a(t)
\qquad (t,x)\in Q.
\]
Define
\[
\alpha(t):=\int_{1/2}^t a(s)\,\mathrm ds,
\qquad t\in(0,1).
\]
Then $\alpha'(t)=a(t)$ on $(0,1)$. The gauge-adjusted potential
\[
\widetilde\Phi(t,x):=\Phi(t,x)-\alpha(t)
\]
induces the same velocity on $(0,1)\times \mathbb R ^d$, and satisfies the Hamilton--Jacobi equation
\[
\partial_t\widetilde\Phi=H_q(\nabla_x\widetilde\Phi)
\qquad\text{on }Q.
\]
Let $c_p(x,y)=\lVert x-y\rVert^p/p$. Fenchel's inequality, applied to the convex conjugate pair $L_p(w)=\lVert w\rVert^p/p$ and $H_q(\xi)=\lVert\xi\rVert^q/q$, reads
\[
L_p(w)+H_q(\xi)+w\cdot\xi\geq0,
\]
with equality if and only if $w=-A_q(\xi)$. Fix $0<\varepsilon<1/2$. For every $x\in E$,
\[
F_\Phi(x)-x
=-A_q\bigl(\nabla_x\widetilde\Phi(t,S_t(x))\bigr),
\qquad t\in(0,1),
\]
so equality holds in Fenchel's inequality along the straight-line trajectory $S_t(x)$. Since $\partial_tS_t(x)=F_\Phi(x)-x$, the chain rule and the Hamilton--Jacobi equation give
\[
\begin{aligned}
\frac{\mathrm d}{\mathrm dt}\widetilde\Phi\bigl(t,S_t(x)\bigr)
&=
\partial_t\widetilde\Phi\bigl(t,S_t(x)\bigr)
+
\nabla_x\widetilde\Phi\bigl(t,S_t(x)\bigr)
\cdot\partial_tS_t(x)\\
&=
H_q\bigl(\nabla_x\widetilde\Phi(t,S_t(x))\bigr)
+
\nabla_x\widetilde\Phi\bigl(t,S_t(x)\bigr)
\cdot\bigl(F_\Phi(x)-x\bigr)\\
&=
-L_p\bigl(F_\Phi(x)-x\bigr)\\
&=
-c_p\bigl(x,F_\Phi(x)\bigr).
\end{aligned}
\]
Therefore, the fundamental theorem of calculus on $[\varepsilon,1-\varepsilon]$ yields
\[
\begin{aligned}
&\widetilde\Phi\bigl(1-\varepsilon,S_{1-\varepsilon}(x)\bigr)
-\widetilde\Phi\bigl(\varepsilon,S_\varepsilon(x)\bigr)\\
&\qquad=
\int_\varepsilon^{1-\varepsilon}
\frac{\mathrm d}{\mathrm dt}
\widetilde\Phi\bigl(t,S_t(x)\bigr)\,\mathrm dt
=-(1-2\varepsilon)c_p\bigl(x,F_\Phi(x)\bigr).
\end{aligned}
\]
Equivalently,
\[
\widetilde\Phi\bigl(\varepsilon,S_\varepsilon(x)\bigr)
-\widetilde\Phi\bigl(1-\varepsilon,S_{1-\varepsilon}(x)\bigr)
=
(1-2\varepsilon)c_p\bigl(x,F_\Phi(x)\bigr)
=
\frac{
\lVert S_{1-\varepsilon}(x)-S_\varepsilon(x)\rVert^p
}{p(1-2\varepsilon)^{p-1}}.
\]

Take finitely many points $x_1,\ldots,x_N\in E$, set
$y_i:=F_\Phi(x_i)$, and let $\sigma$ be any permutation of
$\{1,\ldots,N\}$. For each $i$, define the constant-speed line
\[
\gamma_i(t)
:=
S_\varepsilon(x_i)
+
\frac{t-\varepsilon}{1-2\varepsilon}
\left(
S_{1-\varepsilon}(x_{\sigma(i)})
-
S_\varepsilon(x_i)
\right),
\qquad
t\in[\varepsilon,1-\varepsilon].
\]
Its velocity is
\[
\partial_t\gamma_i(t)
=
\frac{
S_{1-\varepsilon}(x_{\sigma(i)})
-
S_\varepsilon(x_i)
}{
1-2\varepsilon
}.
\]
By the chain rule, the Hamilton--Jacobi equation, and Fenchel's
inequality,
\[
\begin{aligned}
\frac{\mathrm d}{\mathrm dt}
\widetilde\Phi\bigl(t,\gamma_i(t)\bigr)
&=
H_q\bigl(\nabla_x\widetilde\Phi(t,\gamma_i(t))\bigr)
+
\nabla_x\widetilde\Phi(t,\gamma_i(t))
\cdot\partial_t\gamma_i(t)\geq
-L_p\bigl(\partial_t\gamma_i(t)\bigr).
\end{aligned}
\]
Integrating over $[\varepsilon,1-\varepsilon]$ and using the fact that
$\partial_t\gamma_i$ is constant gives
\[
\begin{aligned}
&\widetilde\Phi\bigl(\varepsilon,S_\varepsilon(x_i)\bigr)
-\widetilde\Phi\bigl(
1-\varepsilon,S_{1-\varepsilon}(x_{\sigma(i)})
\bigr)\leq
(1-2\varepsilon)
L_p\bigl(\partial_t\gamma_i\bigr)=
\frac{
\lVert
S_{1-\varepsilon}(x_{\sigma(i)})
-
S_\varepsilon(x_i)
\rVert^p
}{
p(1-2\varepsilon)^{p-1}
}.
\end{aligned}
\]
The preceding equality, invariance of the sum of the terminal potential values under permutation, and these cross-pair inequalities give
\[
\begin{aligned}
\sum_{i=1}^N
\frac{
\lVert S_{1-\varepsilon}(x_i)-S_\varepsilon(x_i)\rVert^p
}{p(1-2\varepsilon)^{p-1}}
&=
\sum_{i=1}^N
\left[
\widetilde\Phi\bigl(\varepsilon,S_\varepsilon(x_i)\bigr)
-\widetilde\Phi\bigl(1-\varepsilon,S_{1-\varepsilon}(x_i)\bigr)
\right]\\
&=
\sum_{i=1}^N
\left[
\widetilde\Phi\bigl(\varepsilon,S_\varepsilon(x_i)\bigr)
-\widetilde\Phi\bigl(1-\varepsilon,S_{1-\varepsilon}(x_{\sigma(i)})\bigr)
\right]\\
&\leq
\sum_{i=1}^N
\frac{
\lVert S_{1-\varepsilon}(x_{\sigma(i)})-S_\varepsilon(x_i)\rVert^p
}{p(1-2\varepsilon)^{p-1}}.
\end{aligned}
\]
As $\varepsilon\to0^+$,
\[
S_\varepsilon(x_i)\longrightarrow x_i,
\qquad
S_{1-\varepsilon}(x_i)\longrightarrow y_i.
\]
Taking the limit in the preceding finite-sum inequality yields
\[
\sum_{i=1}^N c_p(x_i,y_i)
\leq
\sum_{i=1}^N c_p(x_i,y_{\sigma(i)}).
\]
Therefore,
\[
\Gamma:=\left\{(x,F_\Phi(x)):x\in E\right\}
\]
is $c_p$-cyclically monotone. Since $\mathcal L_{\mathrm{term}}(\Phi)=0$, the measure
\[
\pi_\Phi:=(\operatorname{Id},F_\Phi)_\#\mu_0
\]
belongs to $\Pi(\mu_0,\mu_1)$ and is concentrated on $\Gamma$. Moreover,
\[
\begin{aligned}
\int_{\mathbb R^d\times\mathbb R^d}c_p(x,y)\,\pi_\Phi(\mathrm dx,\mathrm dy)
&=\frac{1}{p}\int_{\mathbb R^d}\lVert F_\Phi(x)-x\rVert^p\,\mu_0(\mathrm dx)\\
&\leq\frac{2^{p-1}}{p}\left(\int_{\mathbb R^d}\lVert x\rVert^p\,\mu_0(\mathrm dx)+\int_{\mathbb R^d}\lVert y\rVert^p\,\mu_1(\mathrm dy)\right)<\infty.
\end{aligned}
\]
Since $c_p$ is continuous and finite-valued on $\mathbb R^d\times\mathbb R^d$, and $\pi_\Phi$ has finite $c_p$-cost, the standard sufficiency theorem for $c_p$-cyclically monotone transport plans implies that $\pi_\Phi$ is an optimal coupling \citep{santambrogio2015optimal}. The factor $1/p$ does not affect the optimizer, so $F_\Phi$ is optimal for the cost $\lVert x-y\rVert^p$.

Since $\mu_0$ is absolutely continuous and $p>1$, the optimal transport map for this cost is unique $\mu_0$-almost everywhere. Consequently,
\[
F_\Phi=T_p
\qquad\mu_0\text{-almost everywhere}.
\]
We now verify the dynamic optimality explicitly. Since $\rho_t^\Phi=(z_\Phi(t,\cdot))_\#\mu_0$, the pushforward formula, the straight-line identity, and $F_\Phi=T_p$ $\mu_0$-almost everywhere give
\[
\begin{aligned}
\frac{1}{p}\int_0^1\int_{\mathbb R^d}\lVert v_\Phi(t,y)\rVert^p\,\rho_t^\Phi(\mathrm dy)\,\mathrm dt
&=\frac{1}{p}\int_0^1\int_{\mathbb R^d}\left\lVert v_\Phi\bigl(t,z_\Phi(t,x)\bigr)\right\rVert^p\,\mu_0(\mathrm dx)\,\mathrm dt\\
&=\frac{1}{p}\int_0^1\int_{\mathbb R^d}\lVert T_p(x)-x\rVert^p\,\mu_0(\mathrm dx)\,\mathrm dt\\
&=\frac{1}{p}\int_{\mathbb R^d}\lVert T_p(x)-x\rVert^p\,\mu_0(\mathrm dx)\\
&=\frac{1}{p}W_p^p(\mu_0,\mu_1).
\end{aligned}
\]
The generalized Benamou--Brenier formula states that $W_p^p(\mu_0,\mu_1)/p$ is the minimum of the dynamic action. Since each absolutely continuous trajectory $z_\Phi(\cdot,x)$ satisfies the characteristic equation almost everywhere in time and $\rho_t^\Phi=(z_\Phi(t,\cdot))_\#\mu_0$, the standard characteristic-flow argument implies that $(\rho_t^\Phi,v_\Phi)$ satisfies the continuity equation in the distributional sense. Together with the endpoint conditions $\rho_0^\Phi=\mu_0$ and $\rho_1^\Phi=\mu_1$, the preceding action identity therefore shows that $(\rho_t^\Phi,v_\Phi)$ is a minimizer of the $p$-Benamou--Brenier problem.

Finally, for $\mu_0$-almost every $x$,
\[
z_\Phi(t,x)=(1-t)x+tT_p(x)
\qquad\text{for every }t\in[0,1],
\]
and
\[
v_\Phi\bigl(t,z_\Phi(t,x)\bigr)=T_p(x)-x
\qquad\text{for every }t\in(0,1).
\]

The Benamou--Brenier solution \((\rho^\star,v^\star)\) from part~(1) is induced by the same optimal map \(T_p\), and both velocity fields equal \(T_p(x)-x\) along the common trajectory \((1-t)x+tT_p(x)\), \(\mathrm dt\otimes\mu_0\)-almost everywhere. Since
\[
\rho_t^\Phi
=
\bigl((1-t)\operatorname{Id}+tT_p\bigr)_\sharp\mu_0
=
\rho_t^\star,
\]
the pushforward formula gives
\[
\begin{aligned}
&\int_0^1\int_{\mathbb R^d}
\left\lVert v_\Phi(t,y)-v^\star(t,y)\right\rVert^p
\,\rho_t^\Phi(\mathrm dy)\,\mathrm dt\\
=&
\int_0^1\int_{\mathbb R^d}
\left\lVert
v_\Phi\bigl(t,(1-t)x+tT_p(x)\bigr)
-
v^\star\bigl(t,(1-t)x+tT_p(x)\bigr)
\right\rVert^p
\,\mu_0(\mathrm dx)\,\mathrm dt\\
=&0.
\end{aligned}
\]
Since the integrand on the left-hand side is nonnegative, it follows that
\[
v_\Phi=v^\star
\qquad
\mathrm dt\,\rho_t^\Phi(\mathrm dy)\text{-almost everywhere}.
\]

This proves both exact recovery and velocity-field uniqueness.
\end{proof}


\subsection{The Role and Non-Redundancy of Interior Full Support}
\label{app:counterexample_no_full_support}

Theorem~\ref{thm:pmot_exactness} shows that a zero-loss potential $\Phi$ recovers the optimal transport map provided that $\Phi$ is $C^2$ on $(0,1)\times\mathbb R^d$ and the induced distribution $\rho_t^\Phi$ has full support for every $t\in(0,1)$. The interior $C^2$ regularity ensures that the differential calculations and chain-rule arguments in the proof are valid. A natural question is whether the full-support condition can be removed or replaced by a weaker assumption. The following counterexamples address this question. Counterexample~1 gives a zero-loss potential for which $\rho_t^\Phi$ does not have full support and the induced terminal map is not optimal; in this example, the target measure $\mu_1$ also fails to have full support. Counterexample~2 further shows that requiring $\mu_1$ itself to have full support is still insufficient to guarantee optimality.

Both counterexamples use $p=2$. In this case $v_\Phi=-\nabla_x\Phi$, and hence, with $S_t(x)=(1-t)x+tF_\Phi(x)$,
\[
\nabla_x\Phi\bigl(t,S_t(x)\bigr)+F_\Phi(x)-x
=-\left[v_\Phi\bigl(t,S_t(x)\bigr)-\bigl(F_\Phi(x)-x\bigr)\right].
\]
Thus the potential-matching residual is the negative of the velocity residual used in the former flow-matching formulation. Their squared norms are identical, so the two zero-loss conditions coincide in these examples.

\paragraph{Counterexample 1: failure after removing interior full support.}
Let $d=2$, $p=2$, and $c_2(x,y)=\lVert x-y\rVert^2/2$. Fix $M>\sqrt{8}$ and define
\[
\begin{aligned}
a_-&:=(-2,-M/2), & u_-&:=(1,M), & b_-&:=a_-+u_-=(-1,M/2),\\
a_+&:=(2,M/2),  & u_+&:=(-1,-M), & b_+&:=a_++u_+=(1,-M/2).
\end{aligned}
\]
Choose $0<r<1/4$ and set
\[
A_-:=B(a_-,r/2),
\qquad
A_+:=B(a_+,r/2).
\]
Define
\[
\mu_0
:=
\frac12\operatorname{Unif}(A_-)
+
\frac12\operatorname{Unif}(A_+).
\]
Define $F$ $\mu_0$-almost everywhere by
\[
F(x):=
\begin{cases}
x+u_-, & x\in A_-,\\
x+u_+, & x\in A_+.
\end{cases}
\]
Define
\[
\mu_1:=F_\sharp\mu_0.
\]
Then
\[
\operatorname{supp}(\mu_1)
=
\overline{B(b_-,r/2)}\cup\overline{B(b_+,r/2)}
\neq\mathbb R^2.
\]

We next construct a smooth potential whose flow agrees with $F$ on the support of $\mu_0$. Define the moving centers
\[
c_-(t):=a_-+tu_-,
\qquad
c_+(t):=a_++tu_+,
\qquad t\in[0,1].
\]
Their first coordinates satisfy
\[
\bigl(c_+(t)\bigr)_1-\bigl(c_-(t)\bigr)_1=4-2t\geq2.
\]
Let $\chi\in C_c^\infty(\mathbb R^2)$ satisfy
\[
\chi(\xi)=1\quad\text{for }\lVert\xi\rVert\leq r,
\qquad
\operatorname{supp}(\chi)\subset B(0,2r).
\]
Because $r<1/4$, the two moving cutoff supports remain disjoint for all $t\in[0,1]$. Define
\[
\begin{aligned}
\Phi(t,y)
:={}&-\chi\bigl(y-c_-(t)\bigr)u_-\cdot\bigl(y-c_-(t)\bigr)\\
&-\chi\bigl(y-c_+(t)\bigr)u_+\cdot\bigl(y-c_+(t)\bigr),
\end{aligned}
\]
and let $v_\Phi=-\nabla_y\Phi$. On $B(c_-(t),r)$, the first cutoff is identically one and the second vanishes, so
\[
v_\Phi(t,y)=u_-.
\]
Similarly, $v_\Phi(t,y)=u_+$ on $B(c_+(t),r)$. The potential $\Phi$ is smooth on $[0,1]\times\mathbb R^2$, and $v_\Phi$ is smooth with uniformly compact spatial support. Hence it generates a unique global smooth flow, and every time slice $z_\Phi(t,\cdot)$ is a diffeomorphism of $\mathbb R^2$.

For $x\in A_-$, the curve $z(t)=x+tu_-$ satisfies
\[
\lVert z(t)-c_-(t)\rVert=\lVert x-a_-\rVert<r/2<r,
\]
and therefore
\[
\partial_tz(t)=u_-=v_\Phi\bigl(t,z(t)\bigr).
\]
Uniqueness of the flow gives $z_\Phi(t,x)=x+tu_-$ for every $t\in[0,1]$. The same argument gives $z_\Phi(t,x)=x+tu_+$ for $x\in A_+$. Consequently, for $\mu_0$-almost every $x$,
\[
z_\Phi(t,x)=(1-t)x+tF(x),
\qquad
F_\Phi(x)=z_\Phi(1,x)=F(x).
\]
It follows that
\[
v_\Phi\bigl(t,(1-t)x+tF_\Phi(x)\bigr)=F_\Phi(x)-x
\]
for every $t\in[0,1]$ and for $\mu_0$-almost every $x$. Thus $\mathcal L_{\mathrm{PM}}(\Phi)=0$. Since $(F_\Phi)_\sharp\mu_0=\mu_1$, the terminal loss also vanishes, and hence
\[
\mathcal L_{\mathrm{PMOT}}(\Phi)=0.
\]
Because the PMOT objective is nonnegative, $\Phi$ is a global minimizer. Nevertheless, its intermediate distributions occupy only two moving tubes:
\[
\operatorname{supp}(\rho_t^\Phi)
=
\overline{B(c_-(t),r/2)}\cup\overline{B(c_+(t),r/2)}
\neq\mathbb R^2,
\qquad t\in[0,1].
\]

The induced coupling is not optimal. Since the density of $\mu_0$ is strictly positive near $a_-$ and $a_+$ and $F$ is a translation on each ball, both $(a_-,b_-)$ and $(a_+,b_+)$ belong to the support of
\[
\pi:=(\operatorname{Id},F)_\sharp\mu_0.
\]
The cost of the assigned pairing is
\[
c_2(a_-,b_-)+c_2(a_+,b_+)
=\frac12\lVert u_-\rVert^2+\frac12\lVert u_+\rVert^2
=1+M^2.
\]
After exchanging the endpoints, $a_--b_+=(-3,0)$ and $a_+-b_-=(3,0)$, so
\[
c_2(a_-,b_+)+c_2(a_+,b_-)=9.
\]
Since $M>\sqrt{8}$, we have $1+M^2>9$. Thus the support of $\pi$ violates the two-point $c_2$-cyclical-monotonicity inequality, and $\pi$ is not an optimal coupling between $\mu_0$ and $\mu_1$.

Thus, without interior full support, a smooth potential may attain zero PMOT loss while inducing a nonoptimal terminal map. The obstruction is the persistent vacuum outside the two moving tubes: the potential-matching loss constrains the potential gradient, and equivalently the velocity in the quadratic case, only along mass-carrying trajectories, so the material-derivative identity cannot be extended to the entire interior space--time domain.

\paragraph{Counterexample 2: terminal full support does not replace interior full support.}
In Counterexample~1, the target measure \(\mu_1\) does not have full support. The following example further shows that even the condition \(\operatorname{supp}(\mu_1)=\mathbb R^d\) is insufficient to guarantee that a zero-loss terminal map is optimal when the intermediate distributions fail to have full support.

Let $d=2$, $p=2$, and fix $M>\sqrt{8}$. Define
\[
K_-:=\{x\in\mathbb R^2:x_1<-1\},
\qquad
K_+:=\{x\in\mathbb R^2:x_1>1\},
\]
and set
\[
a_-:=(-2,-M/2),
\qquad
a_+:=(2,M/2).
\]
For suitable normalizing constants $Z_-,Z_+>0$, let
\[
f_-(x):=\frac{1}{Z_-}e^{-\lVert x-a_-\rVert^2}\mathbf 1_{K_-}(x),
\qquad
f_+(x):=\frac{1}{Z_+}e^{-\lVert x-a_+\rVert^2}\mathbf 1_{K_+}(x),
\]
and define
\[
\mu_0(\mathrm dx)
:=
\frac12f_-(x)\,\mathrm dx
+
\frac12f_+(x)\,\mathrm dx.
\]
Then $\mu_0\ll\mathcal L^2$ and $\mu_0\in\mathcal P_2(\mathbb R^2)$.

Let
\[
u_-:=(1,M),
\qquad
u_+:=(-1,-M),
\]
and define
\[
F(x):=
\begin{cases}
x+u_-, & x\in K_-,\\
x+u_+, & x\in K_+,
\end{cases}
\qquad
\mu_1:=F_\sharp\mu_0.
\]
Since $F(K_-)=\{y_1<0\}$ and $F(K_+)=\{y_1>0\}$, we have
\[
\operatorname{supp}(\mu_1)=\mathbb R^2.
\]
For $S_t(x):=(1-t)x+tF(x)$ and $\rho_t:=(S_t)_\sharp\mu_0$,
\[
\operatorname{supp}(\rho_t)
=
\{y_1\leq-(1-t)\}\cup\{y_1\geq1-t\}
\neq\mathbb R^2,
\qquad t\in(0,1).
\]

Choose a smooth nondecreasing function $\vartheta\in C^\infty(\mathbb R;[0,1])$ such that
\[
\vartheta(r)=0\quad\text{for }r\leq-1,
\qquad
\vartheta(r)=1\quad\text{for }r\geq1,
\]
and define, for $(t,y)\in(0,1)\times\mathbb R^2$,
\[
\Phi(t,y)
:=
\left(
2\vartheta\left(\frac{y_1}{1-t}\right)-1
\right)(y_1+My_2).
\]
The values of $\Phi$ at $t=0$ and $t=1$ may be chosen arbitrarily. Then $\Phi\in C^\infty((0,1)\times\mathbb R^2)$. If $y_1<-(1-t)$, then
\[
\Phi(t,y)=-(y_1+My_2),
\qquad
v_\Phi(t,y)=-\nabla_y\Phi(t,y)=(1,M)=u_-.
\]
Similarly, if $y_1>1-t$, then
\[
\Phi(t,y)=y_1+My_2,
\qquad
v_\Phi(t,y)=(-1,-M)=u_+.
\]
For every $x\in K_-\cup K_+$ and $t\in(0,1)$, the straight-line trajectory $S_t(x)$ remains in the corresponding occupied component. Hence
\[
v_\Phi\bigl(t,S_t(x)\bigr)
=F(x)-x
=\partial_tS_t(x).
\]
Thus $F_\Phi=F$ $\mu_0$-almost everywhere, and
\[
\mathcal L_{\mathrm{PM}}(\Phi)=0,
\qquad
\mathcal L_{\mathrm{term}}(\Phi)=0,
\qquad
\mathcal L_{\mathrm{PMOT}}(\Phi)=0.
\]
In particular, $\Phi\in\operatorname{dom}(\mathcal L_{\mathrm{PMOT}})$ and is a global minimizer.

Nevertheless, $F$ is not optimal. Let
\[
x_-:=a_-,
\qquad
x_+:=a_+,
\qquad
y_-:=F(x_-)=(-1,M/2),
\qquad
y_+:=F(x_+)=(1,-M/2).
\]
For $c_2(x,y)=\lVert x-y\rVert^2/2$,
\[
\begin{aligned}
c_2(x_-,y_-)+c_2(x_+,y_+)&=1+M^2,\\
c_2(x_-,y_+)+c_2(x_+,y_-)&=9.
\end{aligned}
\]
Since $M>\sqrt{8}$, the first quantity is larger than the second. Moreover, the density of $\mu_0$ is positive near $x_-$ and $x_+$, so $(x_-,y_-)$ and $(x_+,y_+)$ belong to the support of $(\operatorname{Id},F)_\sharp\mu_0$. Hence this support is not $c_2$-cyclically monotone, and $(\operatorname{Id},F)_\sharp\mu_0$ is not an optimal coupling.

Thus $\mathcal L_{\mathrm{PMOT}}(\Phi)=0$ and $\operatorname{supp}(\mu_1)=\mathbb R^2$, but $F_\Phi$ is not an optimal Monge map. Terminal full support therefore does not replace the interior full-support condition in Theorem~\ref{thm:pmot_exactness}.

\section{Additional Experimental Results}
\label{app:additional_experiments}

\subsection{Additional Toy Generation Results}
\label{app:toy_generation}

Figure~\ref{fig:toy_generation_appendix} shows unconditional generation results
on the two-dimensional toy benchmarks. Samples are generated by transporting
standard Gaussian samples through the learned inverse flow. These results
verify endpoint generation quality, while the main text focuses on transport
geometry and $p$-specific alignment.

\begin{figure}[htbp]
    \centering
    \includegraphics[width=0.8\linewidth]{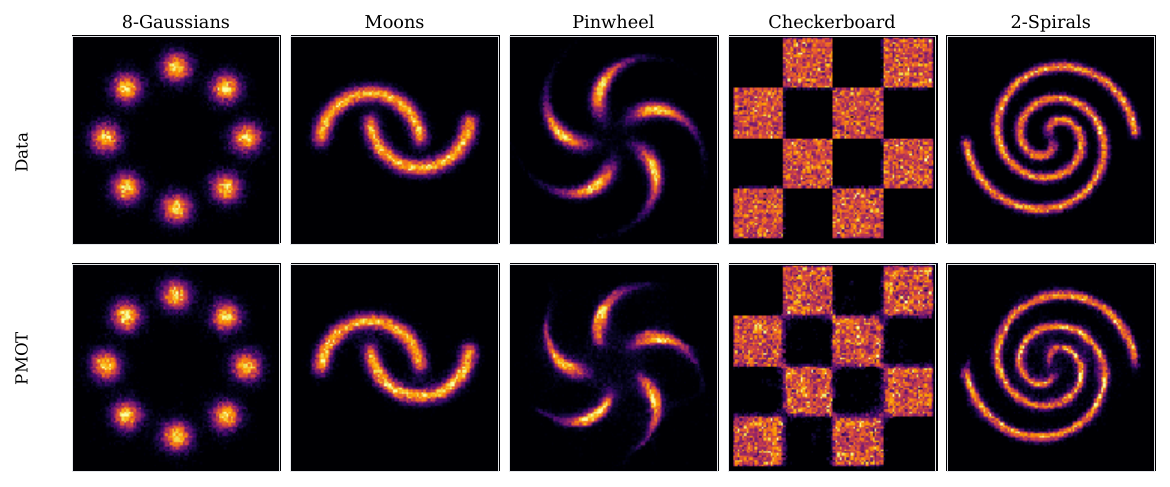}
    \caption{
    Unconditional generation on two-dimensional toy benchmarks.
    The top row shows samples from the target distributions, and the bottom row
    shows samples generated by PMOT from a standard Gaussian prior.
    All panels use the same two-dimensional histogram visualization protocol.
    }
    \label{fig:toy_generation_appendix}
\end{figure}

\subsection{Repeated \texorpdfstring{$p$}{p}-Swap Evaluation on Pinwheel}
\label{app:pinwheel_p_swap_multiseed}

In the main text, Figure~\ref{fig:pinwheel_p_swap_heatmap} shows a representative
$p$-swap evaluation on Pinwheel. To verify that the diagonal preference is not
due to a particular evaluation sample, we repeat the evaluation over five
independently sampled source and reference batches. For each repetition, we
compute the mean endpoint MSE over the diagonal entries, corresponding to
matched training and reference exponents, and over the off-diagonal entries,
corresponding to mismatched exponents.

\begin{table}[t]
\centering
\caption{
Repeated $p$-swap evaluation on Pinwheel. Diagonal entries compare each PMOT
checkpoint with the Sinkhorn barycentric reference computed using the matching
exponent. Off-diagonal entries use mismatched reference exponents. Values are
averaged over the corresponding heatmap entries in each evaluation repetition
and reported as mean $\pm$ standard deviation.
}
\label{tab:pinwheel_p_swap_multiseed}
\vspace{0.5em}
\begin{tabular}{lccc}
\toprule
Evaluations & Diagonal MSE $\downarrow$ & Off-diagonal MSE $\downarrow$ & Off/Diagonal \\
\midrule
5 resamples
& $0.0255 \pm 0.0031$
& $0.0359 \pm 0.0034$
& $1.410 \pm 0.048$ \\
\bottomrule
\end{tabular}
\end{table}

Across the five evaluation repetitions, the mean off-diagonal endpoint MSE is
$1.410\pm0.048$ times the mean diagonal MSE, corresponding to an average
increase of $41.0\%$. This indicates that the fixed PMOT checkpoints are
consistently closer to post-hoc OT references computed with their matching
cost exponents. These repetitions measure evaluation-sample variability rather
than variability across independent training runs.

\subsection{Likelihood-Weight Sensitivity on POWER}
\label{app:power_weight_sensitivity}

Table~\ref{tab:power_weight_sensitivity_appendix} reports an additional
likelihood-weight sensitivity study on POWER. PMOT uses the same fixed weights
$(\lambda_{\mathrm{PM}},\lambda_{\mathrm{NLL}})=(1,5)$ as in the main
high-dimensional experiments and does not include an HJB/R regularizer. For
OT-Flow, we keep the kinetic and HJB/R coefficients fixed and vary only the
likelihood coefficient $\lambda_C$.

\begin{table}[t]
\centering
\caption{
Likelihood-weight sensitivity on POWER. PMOT weights are reported as
$(\lambda_{\mathrm{PM}},\lambda_{\mathrm{NLL}})$, while OT-Flow weights are
reported as $(\lambda_L,\lambda_C,\lambda_R)$. OT-Flow keeps its kinetic and
HJB/R coefficients fixed while varying $\lambda_C$. Evaluation uses biased MMD
with 4000 generated and 4000 held-out samples. Lower is better.
}
\label{tab:power_weight_sensitivity_appendix}
\vspace{0.5em}
\begin{tabular}{llcc}
\toprule
Method & Weights & NLL $\downarrow$ & MMD $\downarrow$ \\
\midrule
PMOT & $(1,5)$ & -0.383 & $2.91{\times}10^{-4}$ \\
\midrule
OT-Flow & $(1,10,5)$ & -0.250 & $3.83{\times}10^{-4}$ \\
OT-Flow & $(1,50,5)$ & -0.287 & $3.91{\times}10^{-4}$ \\
OT-Flow & $(1,500,5)$ & -0.312 & $3.57{\times}10^{-4}$ \\
\bottomrule
\end{tabular}
\end{table}

Increasing the OT-Flow likelihood coefficient improves NLL on POWER, although
generation MMD varies non-monotonically. PMOT achieves lower NLL and MMD with
the fixed $(1,5)$ weighting and no HJB/R regularizer. Together with the
8-Gaussians sensitivity experiment, these results suggest that PMOT remains
effective without a heavily likelihood-dominated objective.

\subsection{Image Color Transformation}
\label{app:color_transform}

We use image color transformation to illustrate PMOT with sample-based terminal
matching in RGB space. In this experiment, the terminal discrepancy in
\eqref{eq:pmot_general_terminal_loss} is instantiated as MMD rather than the
NLL objective used for toy and tabular density estimation. The learned forward
flow transports the empirical color distribution of a content image toward that
of a style image. The resulting map is applied independently to each pixel, so
the spatial layout of the content image is retained.

We select two $256\times256$ images from ImageNet
\citep{deng2009imagenet,russakovsky2015imagenet}: an American chameleon
(synset n01682714) as the content image and a sulphur-crested cockatoo
(synset n01819313) as the style image. Their empirical RGB distributions are
treated as the source and target distributions, respectively. RGB values are
rescaled to $(0,1)^3$ and mapped to $\mathbb R^3$ using a componentwise logit
transform. We train PMOT with $p=2$,
$(\lambda_{\mathrm{PM}},\lambda_{\mathrm{MMD}})=(1,10)$, and a mixture of five
Gaussian kernels with bandwidths
$\sigma\in\{0.25,0.5,1.0,2.0,4.0\}$ for 25,000 iterations.

Table~\ref{tab:color_transform_loss} reports the final training losses.
Figure~\ref{fig:color_transform_transfer_montage} in the main text visualizes
the transformed image along the learned flow. The marginal RGB distributions
in Figure~\ref{fig:marginal_distribution} provide an additional qualitative
comparison between the source, transformed, and target color distributions.

\begin{table}[t]
\centering
\caption{
Final PMOT training losses for image color transformation.
}
\label{tab:color_transform_loss}
\vspace{0.5em}
\begin{tabular}{lcc}
\toprule
Method & $\mathcal L_{\mathrm{PM}}$ & $\mathcal L_{\mathrm{MMD}}$ \\
\midrule
PMOT & $2.891\times10^{-3}$ & $6.664\times10^{-4}$ \\
\bottomrule
\end{tabular}
\end{table}

\begin{figure}[htbp]
    \centering
    \includegraphics[width=\linewidth]
    {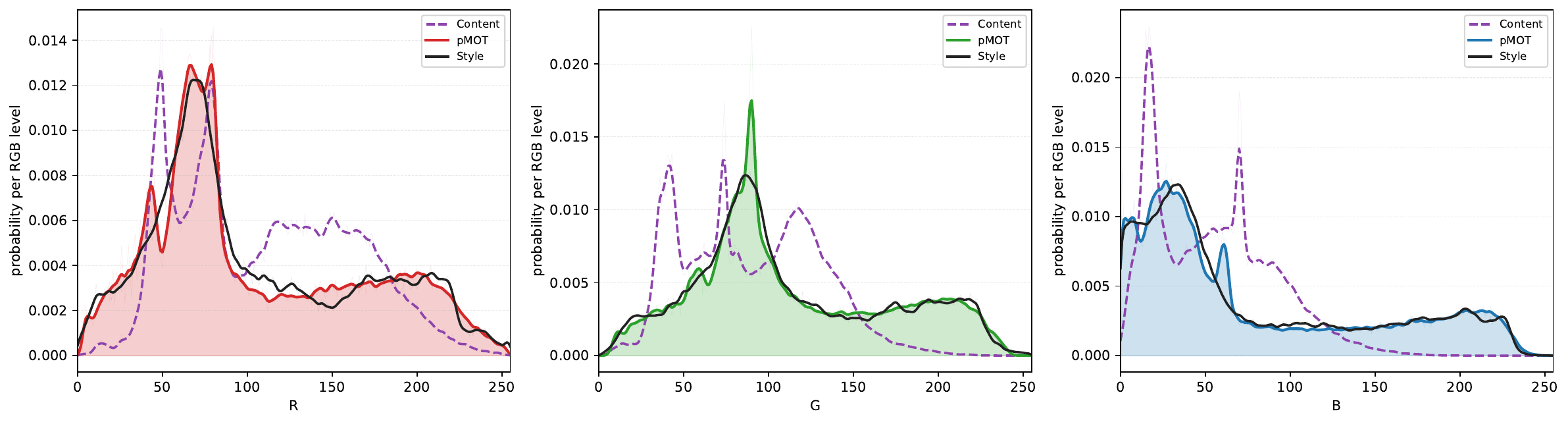}
    \caption{
    Marginal red, green, and blue distributions of the content image, the
    PMOT-transformed image, and the target style image. The transformed
    marginals move toward those of the target, providing qualitative evidence
    of color-distribution matching.
    }
    \label{fig:marginal_distribution}
\end{figure}

\end{document}